\documentclass{article}
\usepackage[margin=1.125in]{geometry}
\usepackage[round,comma]{natbib}
\usepackage{amsmath}
\usepackage{amssymb,graphicx,dsfont}

\usepackage{amsthm}
\usepackage{thmtools}
\usepackage{subfig} 
\usepackage[usenames,dvipsnames]{color}
\usepackage{hyperref}
\usepackage{cleveref}
\usepackage{mathrsfs}
\usepackage{wrapfig}
\usepackage{mathtools}
\usepackage{nicefrac}

\def\ddefloop#1{\ifx\ddefloop#1\else\ddef{#1}\expandafter\ddefloop\fi}
\def\ddef#1{\expandafter\def\csname b#1\endcsname{\ensuremath{\mathbf{#1}}}}
\ddefloop ABCDEFGHIJKLMNOPQRSTUVWXYZ\ddefloop
\def\ddef#1{\expandafter\def\csname bb#1\endcsname{\ensuremath{\mathbb{#1}}}}
\ddefloop ABCDEFGHIJKLMNOPQRSTUVWXYZ\ddefloop
\def\ddef#1{\expandafter\def\csname c#1\endcsname{\ensuremath{\mathcal{#1}}}}
\ddefloop ABCDEFGHIJKLMNOPQRSTUVWXYZ\ddefloop
\def\ddef#1{\expandafter\def\csname v#1\endcsname{\ensuremath{\boldsymbol{#1}}}}
\ddefloop ABCDEFGHIJKLMNOPQRSTUVWXYZabcdefghijklmnopqrstuvwxyz\ddefloop
\def\ddef#1{\expandafter\def\csname
  v#1\endcsname{\ensuremath{\boldsymbol{\csname #1\endcsname}}}}
\ddefloop
    {alpha}{beta}{gamma}{delta}{epsilon}{varepsilon}{zeta}{eta}{theta}{vartheta}{iota}{kappa}{lambda}{mu}{nu}{xi}{pi}{varpi}{rho}{varrho}{sigma}{varsigma}{tau}{upsilon}{phi}{varphi}{chi}{psi}{omega}{Gamma}{Delta}{Theta}{Lambda}{Xi}{Pi}{Sigma}{varSigma}{Upsilon}{Phi}{Psi}{Omega}\ddefloop

\def\1{\mathds 1}
\def\R{\mathbb R}

\def\cRz{\cR_{\textup{z}}}

\def\fN{\mathfrak N}
\def\fR{\mathfrak R}

\newcommand{\ip}[2]{\left\langle #1, #2 \right \rangle}

\numberwithin{equation}{section}
\declaretheorem[numberlike=equation]{theorem}
\declaretheorem[numberlike=theorem]{lemma}

\declaretheoremstyle[%
qed={\ensuremath\Diamond}]{remstyle}

\def\srelu{\sigma_{\textsc{r}}}

\def\fm{f_{\textup{m}}}

\title{Representation Benefits of Deep Feedforward Networks}
\author{Matus Telgarsky}
\date{}

\begin{document}

\maketitle

\begin{abstract}
  This note provides a family of classification problems, indexed by a positive integer $k$,
  where all shallow networks with fewer than exponentially (in $k$) many nodes exhibit error at least $1/6$,
  whereas a deep network with 2 nodes in each of $2k$ layers achieves zero error,
  as does a recurrent network with 3 distinct nodes iterated $k$ times.
  The proof is elementary,
  and the networks are standard feedforward networks with
  ReLU (Rectified Linear Unit) nonlinearities.
\end{abstract}

\section{Overview}

A \emph{neural network} is a function whose evaluation is defined by a graph as follows.
Root nodes compute $x\mapsto \sigma(w_0 + \ip{w}{x})$, where $x$ is the input to the network,
and  $\sigma:\R\to\R$ is typically a nonlinear function, for instance the ReLU (Rectified Linear Unit) $\srelu(z) = \max\{0,z\}$.
Internal nodes perform a similar computation, but now their input vector is the collective output of their parents.
The choices of $w_0$ and $w$ may vary from node to node, and the possible set of functions obtained by
varying these parameters gives the function class $\fN(\sigma; m,l)$, which has $l$ layers each with at most $m$ nodes.

The representation power of $\fN(\sigma; m, l)$ will be measured via the \emph{classification error} $\cRz$.
Namely, given a function $f:\R^d \to \R$, let $\tilde f : \R^d \to \{0,1\}$ denote the corresponding classifier
$\tilde f(x) := \1[f(x) \geq 1/2]$,
and additionally given a sequence of points $((x_i,y_i))_{i=1}^n$ with $x_i\in\R^d$
and $y_i\in\{0,1\}$,
define $\cRz(f) := n^{-1} \sum_i \1[ \tilde f(x_i) \neq y_i ]$.

\begin{theorem}
  \label{fact:gap:simplified}
  Let positive integer $k$,
  number of layers $l$,
  and number of nodes per layer $m$ be given with $m \leq 2^{(k-3)/l - 1}$.
  Then there exists a collection of $n:= 2^k$ points $((x_i,y_i))_{i=1}^n$
  with $x_i\in [0,1]$ and $y\in \{0,1\}$
  such that
  \[
    \min_{f \in \fN(\srelu;2,2k)} \cRz(f) = 0
    \qquad
    \textup{and}
    \qquad
    \min_{g \in \fN(\srelu;m,l)} \cRz(g) \geq \frac 1 {6}.
  \]
\end{theorem}

For example, approaching the error of the $2k$-layer network (which has $\cO(k)$ nodes and weights)
with $2$ layers requires at least $2^{(k-3)/2-1}$ nodes,
and with $\sqrt{k-3}$ layers needs at least $2^{\sqrt{k-3}-1}$ nodes.

The purpose of this note is to provide an elementary proof of \Cref{fact:gap:simplified}
and its refinement \Cref{fact:gap}, which amongst other improvements will use a \emph{recurrent neural network}
in the upper bound.
\Cref{sec:analysis} will present the proof,
and \Cref{sec:discussion} will tie these results to the literature on neural network expressive power and circuit complexity,
which by contrast makes use of product nodes rather than standard feedforward networks when showing the benefits of depth.

\subsection{Refined bounds}

\begin{wrapfigure}{r}{0.4\textwidth}
  \vspace{-24pt}
  \begin{center}
    \includegraphics[width=0.4\textwidth]{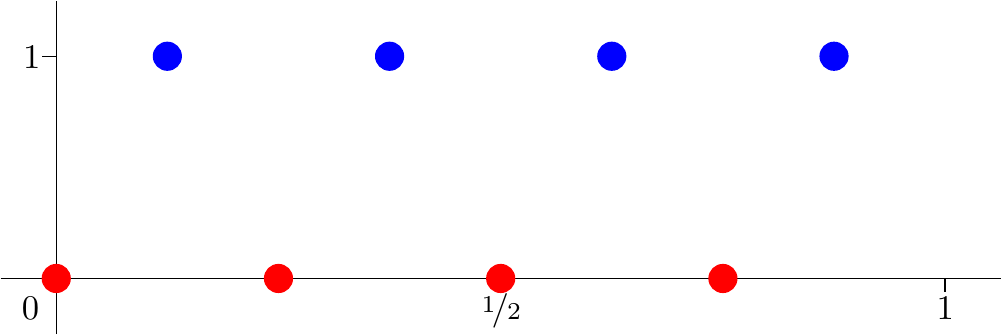}
  \end{center}
   \vspace{-10pt}
  \caption{The $3$-ap.}
  \label{fig:3ap}
  \vspace{-10pt}
\end{wrapfigure}
There are three refinements to make: the classification problem will be specified,
the perfect network will be an even simpler \emph{recurrent} network,
and $\sigma$ need not be $\srelu$.

Let $n$-ap (the \emph{$n$-alternating-point} problem) denote the set of $n$ uniformly spaced points within $[0,1-2^{-n}]$ with alternating labels,
as depicted in \Cref{fig:3ap};
that is, the points $((x_i,y_i))_{i=1}^{n}$ with $x_i = i2^{-n}$, and $y_i = 0$ when $i$ is even, and otherwise $y_i=1$.
As the $x$ values pass from left to right, the labels change as often as possible;
the key
is that adding a constant number of nodes in a flat network only corrects predictions
on a constant number of points, whereas adding a constant number of nodes in a deep network can correct predictions
on a constant \emph{fraction} of the points.

Let $\fR(\sigma;m,l;k)$ denote $k$ iterations of a \emph{recurrent network} with $l$ layers of at most $m$ nodes each, defined as follows.
Every $f\in\fR(\sigma;m,l;k)$ consists of some
fixed network $g\in\fN(\sigma;m,l)$ applied $k$ times:
\[
  f(x) = g^k(x) = \big(\underbrace{g\circ\cdots\circ g}_{\textup{$k$ times}}\big)(x).
\]
Consequently, $\fR(\sigma; m,l;k)\subseteq \fN(\sigma; m, lk)$, but the former has $\cO(ml)$ parameters whereas the latter has $\cO(mlk)$ parameters.

Lastly, say that $\sigma :\R \to \R$ is \emph{$t$-sawtooth} if it is piecewise affine with $t$ pieces,
meaning $\R$ is partitioned into $t$ consecutive intervals, and $\sigma$ is affine within each interval.
Consequently, $\srelu$ is 2-sawtooth, but this class also includes many other functions, for instance the
decision stumps used in boosting are 2-sawtooth,
and decision trees with $t-1$ nodes correspond to $t$-sawtooths.

\begin{theorem}
  \label{fact:gap}
  Let positive integer $k$,
  number of layers $l$,
  and number of nodes per layer $m$ be given.
  Given a $t$-sawtooth $\sigma :\R\to\R$
  and $n := 2^k$ points as specified by the $n$-ap,
  then
  \[
    \min_{f \in \fR(\srelu;2,2;k)} \cRz(f) = 0
    \qquad
    \textup{and}
    \qquad
    \min_{g \in \fN(\sigma;m,l)} \cRz(g) \geq \frac {n - 4(tm)^l}{3n}.
  \]
\end{theorem}
This more refined result can thus say, for example,
that on the $2^k$-ap one needs exponentially (in $k$) many parameters when boosting decision stumps,
linearly many parameters with a deep network,
and constantly many parameters with a recurrent network.

\section{Analysis}
\label{sec:analysis}

This section will first prove the lower bound via a counting argument, simply tracking the number of times a function within $\fN(\sigma;m,l)$
can cross 1/2.
The upper bound will exhibit a network in $\fN(\srelu;2,2)$ which can be composed with itself $k$ times to exactly
fit the $n$-ap.
These bounds together prove \Cref{fact:gap}, which in turn implies \Cref{fact:gap:simplified}.

\subsection{Lower bound}

The lower bound is proved in two stages.
First, composing and summing sawtooth functions must also yield a sawtooth function,
thus elements of $\fN(\sigma;m,l)$ are sawtooth whenever $\sigma$ is.
Secondly, a sawtooth function can not cross $1/2$ very often, meaning it can't hope to match
the quickly changing labels of the $n$-ap.

To start, $\fN(\sigma;m,l)$ is sawtooth as follows.
\begin{lemma}
  \label{fact:sawtooth_props:2}
  If $\sigma$ is $t$-sawtooth, then every $f\in \fN(\sigma;m,l)$ with $f : \R\to\R$ is $(tm)^l$-sawtooth.
\end{lemma}
The proof is straightforward and deferred momentarily.  The key observation is that adding together sawtooth functions
grows the number of regions very slowly, whereas composition grows the number very quickly, an early
sign of the benefits of depth.

Given a sawtooth function, its classification error on the $n$-ap may be lower bounded as follows.
\begin{lemma}
  \label{fact:representation:lb}
  Let $((x_i,y_i))_{i=1}^n$ be given according to the $n$-ap.
  Then every $t$-sawtooth function $f:\R\to\R$ satisfies
  $\cRz(f) \geq (n - 4t)/ (3n)$.
\end{lemma}
\begin{proof}
  Recall the notation $\tilde f(x) := \1[f(x) \geq 1/2]$, whereby $\cRz(f) := n^{-1}\sum_i \1[y_i\neq \tilde f(x_i)]$.
  Since $f$ is piecewise monotonic with a corresponding partition $\R$ having at most $t$ pieces,
  then $f$ has at most $2t-1$ crossings of 1/2: at most one within each interval of the partition, and at most 1
  at the right endpoint of all but the last interval.
  Consequently, $\tilde f$ is piecewise \emph{constant},
  where the corresponding partition of $\R$ is into at most $2t$ intervals.
  This means $n$ points with alternating labels must land in $2t$ buckets,
  thus the total number of points landing in buckets with at least three points
  is at least $n-4t$.
  Since buckets are intervals and signs must alternate within any such interval,
  at least a third of the points in any of these buckets are labeled incorrectly by $\tilde f$.
\end{proof}

To close, the proof of \Cref{fact:sawtooth_props:2} proceeds as follows.
First note how adding and composing sawtooths grows their complexity.

\begin{lemma}
  \label{fact:sawtooth_props}
  Let $f:\R\to\R$ and $g:\R\to\R$ be respectively $k$- and $l$-sawtooth.
  Then $f+g$ is $(k+l)$-sawtooth, and $f\circ g$ is $kl$-sawtooth.
\end{lemma}

\begin{proof}[Proof of \Cref{fact:sawtooth_props}]
  Let $\cI_f$ denote the partition of $\R$ corresponding to $f$, and $\cI_g$ denote the partition
  of $\R$ corresponding to $g$.

  First consider $f+g$, and moreover any intervals $U_f\in \cI_f$ and $U_g\in\cI_g$.
  Necessarily, $f+g$ has a single slope along $U_f\cap U_g$.  Consequently, $f+g$ is $|\cI|$-sawtooth,
  where $\cI$ is the set of all intersections of intervals from $\cI_f$ and $\cI_g$, meaning
  $\cI := \{U_f\cap U_g : U_f\in\cI_f, U_g\in\cI_g\}$.  By sorting the left endpoints of elements of
  $\cI_f$ and $\cI_g$, it follows that $|\cI|\leq k+l$ (the other intersections are empty).

  Now consider $f\circ g$, and in particular consider the image $f(g(U_g))$ for some interval $U_g\in \cI_g$.
  $g$ is affine with a single slope along $U_g$, therefore $f$ is being considered along a single unbroken interval
  $g(U_g)$.  However, nothing prevents $g(U_g)$ from hitting all the elements of $\cI_f$; since $U_g$ was arbitrary,
  it holds that $f\circ g$ is $(|\cI_f|\cdot|\cI_g|)$-sawtooth.
\end{proof}

The proof of \Cref{fact:sawtooth_props:2} follows by induction over layers of $\fN(\sigma;m,l)$.

\begin{proof}[Proof of \Cref{fact:sawtooth_props:2}]
  The proof proceeds by induction over layers, showing the output of each node in layer $i$ is $(tm)^i$-sawtooth
  as a function of the neural network input.
  For the first layer, each node starts by computing $x\mapsto w_0 + \ip{w}{x}$, which is itself affine
  and thus 1-sawtooth, so the full node computation $x\mapsto \sigma(w_0 + \ip{w}{x})$ is $t$-sawtooth by \Cref{fact:sawtooth_props}.
  Thereafter, the input to layer $i$ with $i>1$ is a collection of functions $(g_1,\ldots,g_{m'})$ with $m'\leq m$ and $g_j$ being $(tm)^{i-1}$-sawtooth
  by the inductive hypothesis; consequently, $x\mapsto w_0 + \sum_j w_j g_j(x)$ is $m(tm)^{i-1}$-sawtooth by \Cref{fact:sawtooth_props},
  whereby applying $\sigma$ yields a $(tm)^{i}$-sawtooth function (once again by \Cref{fact:sawtooth_props}).
\end{proof}

\subsection{Upper bound}

\begin{wrapfigure}{R}{0.4\textwidth}
  \vspace{-23pt}

  \includegraphics[width=0.4\textwidth]{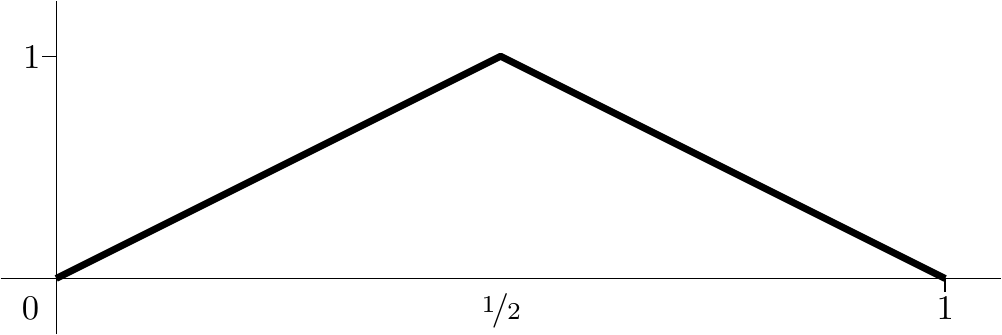}

  \vspace{10pt}

  \includegraphics[width=0.4\textwidth]{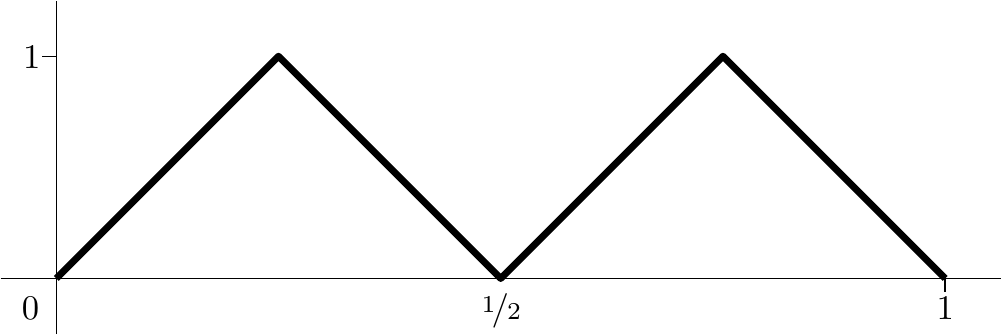}

  \vspace{10pt}

  \includegraphics[width=0.4\textwidth]{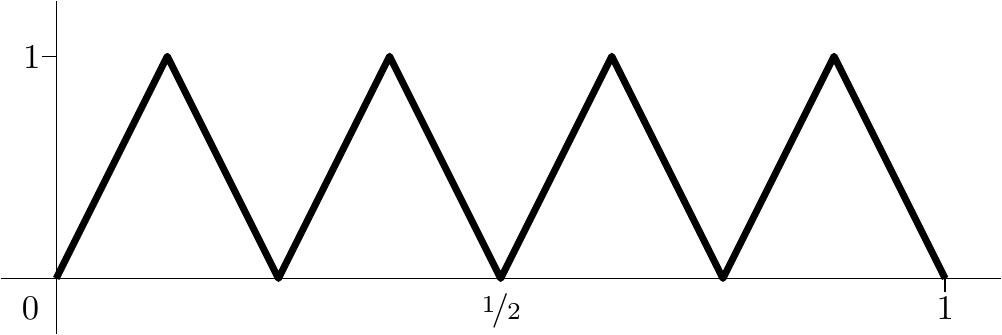}

  \vspace{0pt}

  \caption{$\fm$, $\fm^2$, and $\fm^3$.}
  \label{fig:fm}
  \vspace{-30pt}
\end{wrapfigure}
Consider the \emph{mirror map} $\fm :\R\to\R$, depicted in \Cref{fig:fm},
and defined as
\[
  \fm(x)
  :=
  \begin{cases}
    2x
    &\textup{when $0 \leq x \leq 1/2$},
    \\
    2(1-x)
    &\textup{when $1/2 < x \leq 1$},
    \\
    0
    &\textup{otherwise}.
  \end{cases}
\]
Note that $\fm\in \fN(\srelu;2,2)$; for instance, $\fm(x) = \srelu(2\srelu(x) - 4\srelu(x-1/2))$.
The upper bounds will use $\fm^k \in \fR(\srelu;2,2;k) \subseteq \fN(\srelu;2,2k)$.

To assess the effect of the \emph{post-composition} $\fm \circ g$ for any $g:\R\to\R$, note that $\fm\circ g$ is $2g(x)$ whenever
$g(x) \in[0,1/2]$, and $2(1-g(x))$ whenever $g(x) \in (1/2,1]$.
Visually, this has the effect of reflecting (or folding) the graph of $g$ around the
horizontal line through 1/2 and then rescaling by 2.
Applying this reasoning to $\fm^k$ leads to $\fm^2$ and $\fm^3$ in \Cref{fig:fm}, whose peaks and troughs match the $2^2$-ap and $2^3$-ap,
and moreover have the form of a piecewise affine approximations to sinusoids; indeed,
it was suggested before, by \citet{bengio_lecun_fft},
that Fourier transforms are efficiently represented with deep networks.

These compositions may be written as follows.

\begin{lemma}
  \label{fact:fmk}
  Let real $x \in [0,1]$ and positive integer $k$ be given,
  and choose the unique nonnegative integer $i_k \in \{0,\ldots,2^{k-1}\}$ and real $x_k \in [0, 1)$
  so that $x = (i_k + x_k)2^{1-k}$.
  Then
  \[
    \fm^k(x) = \begin{cases}
      2x_k
      &\textup{when $0 \leq x_k \leq 1/2$},
      \\
      2(1 - x_k)
      &\textup{when $1/2 < x_k < 1$}.
    \end{cases}
  \]
\end{lemma}

In order to prove this form and develop a better understanding of $\fm$,
consider its \emph{pre-composition} behavior $g\circ \fm$ for any $g:\R\to\R$.
Now, $(g\circ \fm)(x) = g(2x)$ whenever $x\in [0,1/2]$, but
$(g\circ \fm)(x) = g(2-2x)$ when $x\in (1/2,1]$; whereas post-composition reflects around
the horizontal line at 1/2 and then scales vertically by 2,
pre-composition first scales horizontally by $1/2$ and then reflects around the vertical line at 1/2,
providing a condensed mirror image and motivating the name \emph{mirror map}.

\begin{proof}[Proof of \Cref{fact:fmk}]
  The proof proceeds by induction on the number of compositions $l$.
  When $l=1$, there is nothing to show.
  For the inductive step,
  the mirroring property of pre-composition with $\fm$
  combined with the symmetry of $\fm^l$ (by the inductive hypothesis)
  implies that every $x\in [0,1/2]$ satisfies
  \[
    (\fm^{l} \circ f)(x) = (\fm^l \circ f)(1 - x) = (\fm^l \circ f)(x + 1/2).
  \]
  Consequently, it suffices to consider $x \in [0,1/2]$,
  which by the mirroring property means $(\fm^l \circ \fm)(x) = \fm^l(2x)$.
  Since the unique nonnegative integer $i_{l+1}$ and real $x_{l+1}\in [0,1)$ satisfy
  $2x = 2(i_{l+1} + x_{l+1})2^{-l-1} = (i_{l+1} + x_{l+1})2^{-l}$,
  the inductive hypothesis applied to $2x$ grants
  \[
    (\fm^l \circ f)(x) = \fm^l(2x) = \begin{cases}
      2x_{l+1}
      &\textup{when $0 \leq x_{l+1} \leq 1/2$},
      \\
      2(1 - x_{l+1})
      &\textup{when $1/2 < x_{l+1} < 1$},
    \end{cases}
  \]
  which completes the proof.
\end{proof}

Before closing this subsection, it is interesting to view $\fm^k$ in one more way,
namely its effect on $((x_i,y_i))_{i=1}^n$ provided by the $n$-ap with $n:=2^k$.
Observe that $((\fm(x_i),y_i))_{i=1}^n$ is an $(n/2)$-ap with all points duplicated except $x_1 = 0$,
and an additional point with $x$-coordinate $1$.

\subsection{Proof of \Cref{fact:gap,fact:gap:simplified}}

It suffices to prove \Cref{fact:gap}, which yields \Cref{fact:gap:simplified} since $\srelu$ is 2-sawtooth,
whereby the condition
$m \leq 2^{(k-3)/l - 1}$
implies
\[
  \frac {n - 4(2m)^l}{3n}
  = \frac 1 3 - (2m)^l2^{-k} \left(\frac 4 3 \right)
  \geq \frac 1 3 - 2^{k-3} 2^{-k} \left(\frac 4 3 \right)
  = \frac 1 3 - \frac 1 6,
\]
and the upper bound transfers since $\fR(\srelu;2,2;k)\subseteq \fN(\srelu;2,2k)$.

Continuing with \Cref{fact:gap}, any $f\in \fN(\sigma;m,l)$ is $(tm)^l$-sawtooth by \Cref{fact:sawtooth_props:2},
whereby \Cref{fact:representation:lb} gives the lower bound.  For the upper bound, note that $\fm^k\in \fR(\srelu;2,2;k) \subseteq \fN(\srelu;2,2k)$
by construction, and moreover $\fm^k(x_i) = \tilde{\fm^k}(x_i) =y_i$ on every $(x_i,y_i)$ in the $n$-ap by \Cref{fact:fmk}.

\section{Related work}
\label{sec:discussion}

The standard classical result on the representation power of neural networks is due to \citet{cybenko},
who proved that neural networks can approximate continuous functions over $[0,1]^d$ arbitrarily well.
This result, however, is for flat networks.

An early result showing the benefits of depth is due to \citet{hastad_thesis},
who established, via an incredible proof, that boolean circuits consisting only of and gates and or gates
require exponential size in order to approximate the parity function well.
These gates correspond to multiplication and addition over the boolean domain, and moreover the parity
function is the Fourier basis over the boolean domain; as mentioned above, $\fm^k$ as used here
is a piecewise affine approximation of a Fourier basis,
and it was suggested previously by \citet{bengio_lecun_fft} that Fourier transforms admit
efficient representations with deep networks.
Lastly, note that \citet{hastad_thesis}'s
work has one of the same weaknesses as the present result, namely of only controlling a countable
family of functions which is in no sense dense.

More generally, networks consisting of sum and product nodes, but now over the reals, have been studied in the
machine learning literature, where it was showed by \citet{bengio_sumproduct_separation} that again there is
an exponential benefit to depth.  While this result was again for a countable class of functions,
more recent work by \citet{tensor_deep_repr_power} aims to give a broader characterization.

Still on the topic of representation results, there is a far more classical result which deserves mention.
Namely, the surreal result of \citet{kolmogorov_nn} states
that a continuous function $f:[0,1]^d\to\R$ can be \emph{exactly} represented by a network
with $\cO(d^2)$ nodes in 3 layers;
this network needs multiple distinct nonlinearities and therefore is not an element of $\fN\left(\sigma; \cO(d^2), 3\right)$ for a fixed $\sigma$,
however one can treat these specialized nonlinearities as goalposts for other representation results.
Indeed, similarly to the $\fm^k$ used here,
\citeauthor{kolmogorov_nn}'s nonlinearities have fractal structure.

Lastly, while this note was only concerned with finite sets of points, it is worthwhile to mention the relevance of
representation power to statistical questions.  Namely, by the seminal result of \citet[Theorem 8.14]{anthony_bartlett_nn},
the VC dimension of $\fN(\srelu;m,l)$ is at most $\cO(m^8l^2)$, indicating that these exponential representation benefits
directly translate into statistical savings.
Interestingly, note that $\fm^k$ has an exponentially large Lipschitz constant (exactly $2^k$), and thus
an elementary statistical analysis via Lipschitz constants and Rademacher complexity \citep{bartlett_mendelson_rademacher}
can inadvertently erase the benefits of depth as presented here.

\addcontentsline{toc}{section}{References}
\bibliographystyle{plainnat}
\bibliography{../sawtooth}

\begin{thebibliography}{8}
\providecommand{\natexlab}[1]{#1}
\providecommand{\url}[1]{\texttt{#1}}
\expandafter\ifx\csname urlstyle\endcsname\relax
  \providecommand{\doi}[1]{doi: #1}\else
  \providecommand{\doi}{doi: \begingroup \urlstyle{rm}\Url}\fi

\bibitem[Anthony and Bartlett(1999)]{anthony_bartlett_nn}
Martin Anthony and Peter~L. Bartlett.
\newblock \emph{Neural Network Learning: Theoretical Foundations}.
\newblock Cambridge University Press, 1999.

\bibitem[Bartlett and Mendelson(2002)]{bartlett_mendelson_rademacher}
Peter~L. Bartlett and Shahar Mendelson.
\newblock Rademacher and gaussian complexities: Risk bounds and structural
  results.
\newblock \emph{JMLR}, 3:\penalty0 463--482, Nov 2002.

\bibitem[Bengio and Delalleau(2011)]{bengio_sumproduct_separation}
Yoshua Bengio and Olivier Delalleau.
\newblock Shallow vs. deep sum-product networks.
\newblock In \emph{NIPS}, 2011.

\bibitem[Bengio and {LeCun}(2007)]{bengio_lecun_fft}
Yoshua Bengio and Yann {LeCun}.
\newblock Scaling learning algorithms towards {AI}.
\newblock In {L{\'{e}}on} Bottou, Olivier Chapelle, D.~DeCoste, and J.~Weston,
  editors, \emph{Large Scale Kernel Machines}. MIT Press, 2007.

\bibitem[Cohen et~al.(2015)Cohen, Sharir, and Shashua]{tensor_deep_repr_power}
Nadav Cohen, Or~Sharir, and Amnon Shashua.
\newblock On the expressive power of deep learning: A tensor analysis.
\newblock 2015.
\newblock {\tt arXiv:1509.05009 [cs.NE]}.

\bibitem[Cybenko(1989)]{cybenko}
George Cybenko.
\newblock {Approximation by superpositions of a sigmoidal function}.
\newblock \emph{Mathematics of Control, Signals and Systems}, 2\penalty0
  (4):\penalty0 303--314, 1989.

\bibitem[H{\r a}stad(1986)]{hastad_thesis}
Johan H{\r a}stad.
\newblock \emph{Computational Limitations of Small Depth Circuits}.
\newblock PhD thesis, Massachusetts Institute of Technology, 1986.

\bibitem[Kolmogorov(1957)]{kolmogorov_nn}
Andrey~Nikolaevich Kolmogorov.
\newblock On the representation of continuous functions of several variables by
  superpositions of continuous functions of one variable and addition.
\newblock 114:\penalty0 953--956, 1957.

\end{thebibliography}
\appendix

\end{document}